%% file: paper.tex
\documentclass[11pt,journal,compsoc,onecolumn]{IEEEtran}
\usepackage{cite}
\usepackage{amsmath,graphicx}
\usepackage{color}
 \usepackage{amssymb}
\usepackage{amsfonts}
\usepackage{dsfont}
\usepackage{epsfig}
 \usepackage{times}
\usepackage{graphicx}
\usepackage{adjustbox}
\usepackage{slashbox}
\usepackage{nicefrac}       
\usepackage{microtype}       
\usepackage{graphicx}
\usepackage{multirow}

\ifCLASSINFOpdf
\else
\fi
\ifCLASSOPTIONcompsoc
 \usepackage[font=normalsize,labelfont=sf,textfont=sf]{subfig}
\else
 \usepackage[font=footnotesize]{subfig}
\fi
\ifCLASSOPTIONcompsoc
 \usepackage[font=normalsize,labelfont=sf,textfont=sf]{subfig}
\else
 \usepackage[font=footnotesize]{subfig}
\fi

\newtheorem{definition}{Definition}
\newtheorem{proposition}{Proposition}
\newtheorem{proof}{Proof}

\begin{document}

 \title{Totally Deep Support Vector Machines}
\author{Hichem Sahbi\\
Sorbonne University, CNRS, LIP6\\
F-75005, Paris, France\\
{\tt\small hichem.sahbi@sorbonne-universite.fr}
}

\maketitle

\begin{abstract}
 Support vector machines (SVMs) have been successful in solving many computer vision tasks including image and video category recognition especially for small and mid-scale training problems. The principle of these non-parametric models is to learn hyperplanes that separate data belonging to different classes while maximizing their margins.  However,  SVMs constrain the learned hyperplanes to lie in the span of support vectors, fixed/taken from training data, and this reduces their representational power  and may lead to limited generalization performances.\\      
In this paper, we relax this constraint and allow the support vectors to be learned (instead of being fixed/taken from training data) in order to better fit a given classification task. Our approach, referred to as deep total variation support vector machines, is parametric and relies on a novel deep architecture that learns not only the   SVM and the kernel parameters but also the support vectors, resulting into highly effective   classifiers. We also show (under a particular setting of the activation functions in this deep architecture) that  a large class of kernels and their combinations can be  learned. Experiments conducted on the challenging task of skeleton-based action recognition show the outperformance of our  deep total variation SVMs w.r.t different baselines as well as the related work.  
\end{abstract}

 \section{Introduction}
Support vector machine  (SVM) --- also known as support vector network --- has been a mainstream standard in machine learning for a while \cite{burges1998tutorial,chang2011libsvm,sahbiaccv2010,sahbipr2012,the1,the2,sahbipr2012,the3,the4,sahbicassp11,sahbiicassp16b,the5,sahbiphd,sahbiicassp15,the6,the7,sahbipr19,the8,the9,sahbiigarss16,the10,the11} before deep learning   has been re-popularized by the seminal work of ~\cite{krizhevsky2012imagenet} and  \cite{deng2014deep, girshick2014rich,Goodfellowetal2016,he2016deep,srivastava2015training,szegedy2015going,russakovsky2015imagenet,tristan2017}. The general idea of SVMs \cite{ShaweTaylor2004} is to learn hyperplanes  that separate populations of training data belonging to different classes while maximizing their margin. These models have been successfully applied to different pattern recognition tasks including image category recognition especially  for small or mid-scale training problems  (see for instance \cite{sahbikpca06,the13,sahbicbmi08,the15,TP19,the16,sahbiclef08,the17,sahbijstars17,the18,sahbispie2004,the19,sahbiiccv17,the20,sahbiigarss11,the21,sahbispie05,the22,sahbiicip18,voc2010,Villegas2013,sahbicassp11,sahbiclef13,sahbijmlr06,sahbispie05,sahbiclef08}).  However, the success of SVMs is highly dependent on the appropriate choice of kernels. The latter, defined as symmetric positive semi-definite functions,  should reserve large values to   highly similar content  and vice-versa \cite{sahbicvpr08a,sahbiicip09}. Existing  kernels are either handcrafted (such as gaussian, histogram intersection, etc. \cite{genton2001classes,Weinberger2004,lingsahbieccv2014,lingsahbi2013,lingsahbiicip2014})  or trained using multiple kernels \cite{Qi2007,Lanckriet2004,Rakotomamonjy2008,Sonnenburg2006,Cortes2009} and explicit kernel maps \cite{Rahimi2007, LiIonescu2010,sahbiicpr18,IosifidisPR2016,sahbiicvs13,Vedaldi2012,SanchezPR2018,HuangICASSP2013}
 as well as their deep variants\footnote{In relation to kernel (or similarity) design, siamese networks also aim to learn functions between pairs of data \cite{taigman2014deepface}, but these similarities are not guaranteed to be positive semi-definite and hence cannot always be plugged into kernel-based SVMs for classification.}   \cite{Yu2009nips,Corinna2010,cho2009kernel,jose2013local,mairal2014convolutional,strobl2013deep,wilson2016deep,icassp2017b}. \\ 

\noindent SVMs are basically non-parametric models;   their training consists in solving a constrained quadratic programming (QP) problem \cite{burges1998tutorial} whose parameter set grows as the size of  training data increases\footnote{Each parameter corresponds to a Lagrange multiplier associated to a given training sample.}. The solution of a given QP (the separating hyperplane) is defined in the span of a {\it  subset} of training data known as the {\it support vectors}, i.e., as a linear combination of training  data whose Lagrange multipliers are strictly positive.  Constraining the SVM solution in the span of fixed support vectors\footnote{In this paper, fixed support vectors do not mean independent from the training samples, but taken from a  finite collection of these samples for which the parameters (denoted $\{\alpha_i\}_i$) are non zeros (see later Eq \ref{representer}).} contributes in making the QP  convex  and always convergent to a global solution regardless of its initialization; however, this reduces the representational power of the learned SVMs as only a subclass of possible functions is explored during optimization (i.e., only those defined  in the span of fixed support vectors), and this makes the estimation risk of  SVMs structurally high compared to other models (including deep learning ones  \cite{hanin17}), especially when the support vectors are not sufficiently representative of the actual  distribution of training and test data.  \\ 

In this paper, we introduce an  alternative   learning algorithm referred to as total variation support vector machine (TV SVM) where the support vectors, kernels and their combinations are  {\it  all} allowed to vary  (together with the original SVM parameters)  resulting into more flexible and highly discriminant classifiers.  Our  model is parametric and  based on  a novel deep network architecture that includes three major steps: (i) support vector learning as a part of  individual kernel design, (ii) kernel combination and (iii) SVM parameter learning.  We will show, for  a particular setting of the activation functions in this deep architecture, that a large  class of   kernels (including distance and inner product-based as well as their combinations) can be modeled. The non-convexity of the underlying deep learning problem makes the class of possible  solutions larger compared to the ones obtained using standard (convex   and non-parametric) SVMs. In contrast to the latter, the VC-dimension~\cite{vapnik:slt} of  our parametric TV SVMs is finite\footnote{The VC-dimension is the maximum number of data samples, that can be shattered, whatever their labels.};  according to Vapnik's VC-theory ~\cite{vapnik:transduction}, the finiteness of the VC-dimension avoids loose generalization bounds, reduces the risk of over-fitting and guarantees better performances as also shown in our experiments. \\   

\noindent In this proposed framework,  the learned support vectors  act as kernel parameters and make it possible to map input data to multiple kernel features prior to their classification (as also achieved in \cite{Rahimi2007, LiIonescu2010,IosifidisPR2016,Vedaldi2012,SanchezPR2018,HuangICASSP2013}). Nevertheless,  the proposed framework is conceptually  different from these related methods; the latter consider the support vectors fixed/taken from training data in order to design explicit kernel maps prior to learn parametric SVMs  while in our method, the support vectors are optimized  to better fit the task at hand. Our method is also different from the reduced set technique \cite{burges1997improving} and the deep kernel nets  (for instance \cite{cho2009kernel,jose2013local,strobl2013deep,wilson2016deep,icassp2017b}). Indeed, the latter operate on precomputed kernel inputs (gram matrices) while the former aims at  reducing the complexity of pre-trained nonlinear SVM classifiers  using a reduced set of virtual support vectors obtained by minimizing a least squares criterion between the initial and newly generated hyperplane classifiers. Furthermore, the reduced set technique assumes that the initial SVMs are already pre-trained (which could be intractable for large scale problems); besides, the newly generated classifiers could be  biased especially when the original pre-trained SVMs are highly complex and nonlinear.  Finally, resulting from the parametric setting of our approach, its  training and testing   complexity is  a priori controllable.   \\ 

The rest of this paper is organized as follows; section \ref{section2}  reminds  the general formulation of SVMs and their  deep  extensions using   kernel networks. Section  \ref{section3} introduces our main contribution;  a total variation SVM that makes it possible to learn not only kernels and SVMs but also the support vectors. Section  \ref{section4} shows the validity of our method through extensive experiments  on the challenging task of skeleton-based action recognition. Section  \ref{section5} concludes the paper while providing possible extensions for a future work.

\def\A{{\bf A}}
\def\Ak{{\bf A}^k}
\def\I{{\bf I}}
\def\X{{\bf X}}
\def\B{{\bf B}}
\def\K{{\bf K}}
\def\U{{\bf U}}
\def\W{{\bf W}}

\def\S{{\cal S}}
\def\N{{\cal N}}

\def\G{{\cal G}}
\def\V{{\cal V}}
\def\E{{\cal E}}
\def\F{{\cal F}}
\def\y{{\bf y}}
\def\x{{\bf x}}
\def\z{{\bf z}}

\section{Deep   SVM networks} \label{section2}
Define ${\cal X} \subseteq \mathbb{R}^p$ as an input space corresponding to all the possible data (e.g., images or videos)   and let ${\cal S}=\{\mathbf{x}_1,\dots,\mathbf{x}_{n},\dots,\mathbf{x}_{n+m}\}$ be a finite subset of $\cal X$ with an arbitrary order. This order is defined so only the first $n$ labels of $\cal S$, denoted $\{\y_1,\dots,\y_{n}\}$,  are known (with $\y_i \in \{-1,+1\}$).
\subsection{Hinge loss max-margin inference at a glance} 

Max margin inference aims at building a decision function $f$ that predicts a label $\y$ for any given input data $\mathbf{x}$; this function is trained on ${\cal L}=\{\mathbf{x}_1,\dots,\mathbf{x}_{n}\}$ and used in order to infer labels on ${\cal U}={\cal S} \backslash {\cal L}$.  In the max-margin classification \cite{vapnik:slt}, we consider ${\phi}$ as a mapping of the input data (in $\cal X$) into a high dimensional space $\mathcal{H}$. The dimension of  $\cal H$  is usually sufficiently large (possibly infinite) in order to guarantee linear separability of data.

\indent Assuming data linearly separable in $\mathcal{H}$, the hinge loss max-margin learning finds a hyperplane $f$ (with a normal $\mathbf{w}$ and shift $b$) that separates $n$ training samples $\{(\mathbf{x}_i,\y_i)\}_{i=1}^{n}$ while maximizing their margin. The margin is defined as twice the distance between the closest training samples w.r.t $f$ and the optimal  $(\hat{\mathbf{w}},\hat{b})$  corresponds to 
\begin{equation}
\begin{array}{cl}
\underset{\mathbf{w},b}{\text{argmin}} &\displaystyle   {\frac{1}{2}}\left\Vert \mathbf{w}\right\Vert_2 ^{2} + C \sum_{i=1}^n \ell(\y_i,f(\x_i)),
\end{array}
\label{eq:hard-margin}
\end{equation}
\noindent which is the primal form of the max-margin support vector machine, $\|.\|_2^2$ is the $\ell_2$-norm, $C>0$ and $\ell(\y_i,f(\x_i))$ is the hinge loss function defined as $\max(0,1-\y_i f(\x_i))$; a differentiable (and also convex) surrogate of this loss is used in practice and defined as $\log(1+\exp(.))$. Given $\mathbf{x}_i \in {\cal U}$, the class of $\mathbf{x}_i$ in $\{-1,+1\}$ is decided by the sign of $f(\mathbf{x}_i)  = {\bf w}'{{\phi}(\mathbf{x}_i)} +b$ with ${\bf w}'$ being the transpose of ${\bf w}$. Following the kernel trick and the representer theorem~\cite{vapnik:slt}, one  may write the solution $\bf w$ of the above problem as 
\begin{equation}\label{representer}
  \begin{array}{l} 
    {\bf w}=\displaystyle \sum_{j=1}^n \alpha_j \y_j \phi( \mathbf{x}_j), 
    \end{array} 
\end{equation} 
hence  $f(\mathbf{x}_i)$ can also be expressed  as  $\sum_{j=1}^{n} \alpha_j \y_j\kappa(\mathbf{x}_i,\mathbf{x}_j) + b$, where $\alpha=(\alpha_1\dots \alpha_{n})'$ is a vector of positive real-valued training parameters found as the solution of the following dual problem
\begin{equation}
\begin{array}{lll}
  \displaystyle & \displaystyle \min_{\alpha \geq 0,b} &  \displaystyle   \frac{1}{2} \sum_{i,j=1}^n \alpha_i \alpha_j  \y_i \y_j \kappa(\mathbf{x}_i,\mathbf{x}_j) \\ 
                                                                                   & & \ \ \ \ +   C  \displaystyle \sum_{i=1}^n \ell(\y_i,\sum_{j=1}^n \alpha_j \y_j\kappa(\mathbf{x}_i,\mathbf{x}_j) + b), 
\end{array}
\label{eq:dual}
\end{equation}
here  $\kappa(.,.) = \langle {\phi}(.),{\phi}(.) \rangle$ is a  symmetric and positive (semi-definite)  kernel function \cite{scholkopf:learnwithkernels,ShaweTaylor2004}. The closed form of $\kappa$ is defined among a collection of existing elementary (a.k.a individual) kernels including linear, gaussian and histogram intersection and the underlying mapping ${\phi}(x) \in {\cal H}$ is usually {\it implicit}, i.e., it does exist but it is not necessarily known and may be infinite dimensional. We consider in the remainder of this paper  an approach that learns better kernels;   the latter are    {\it deep} and designed in order to i) guarantee linear separability of data in $\cal L$, ii) to ensure better generalization performance using deep networks and iii) to ensure positive definiteness by construction (see subsequent sections and also appendix).

\def\w{{\bf \beta}}
\subsection{Deep multiple kernels}\label{deepmkl} 
We aim to learn an implicit mapping function that recursively characterizes a nonlinear and deep combination of multiple elementary kernels \cite{mklimage2017}. For each layer $l \in \{2,\dots,L\}$ and its associated unit $p$, a kernel domain $\big\{ {\kappa}_{p}^{(l)}(\cdot,\cdot) \big\}$ is recursively defined as
\begin{equation}\label{eq0}
{\kappa}_{p}^{l}(\cdot,\cdot) = g \big(\sum_q \w_{q,p}^{l-1} \
\kappa_q^{l-1}(\cdot,\cdot) \big),
\end{equation}
\noindent where $g$ is a nonlinear activation function chosen in order to guarantee the positive semi-definiteness of the learned deep kernels (see more details in appendix). In the above equation, $q \in \{1,\dots,n_{l-1}\}$, $n_{l-1}$ is the number of units in layer $(l-1)$ and $\{\w^{l-1}_{q,p}\}_q$ are the (learned) weights associated to kernel ${\kappa}_{p}^{l}$. In particular, $\{{\kappa}_{p}^{1}\}_p$ are the input elementary kernels including linear and histogram intersection kernels, etc. When $L=2$, the architecture is shallow, and it is equivalent to the 2-layer MKL of Zhuang et al.~\cite{Zhuang2011a}. For larger values of $L$, the network becomes deep.
 \begin{figure}[tb]
 \begin{minipage}[b]{1.0\linewidth}
   \centering
   \includegraphics[width=6cm]{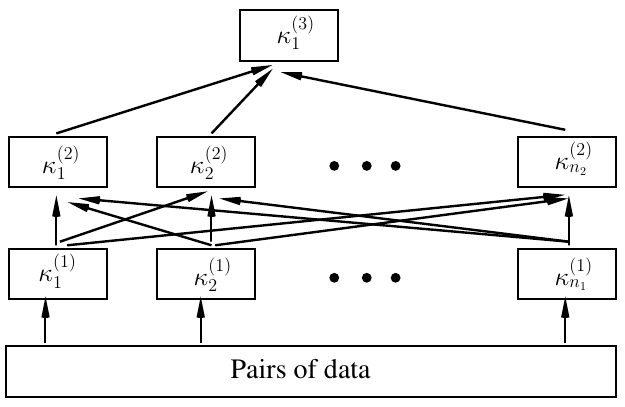}
 \end{minipage}
 \caption{ The deep kernel learning architecture with two  intermediate layers and an output layer. The input of this network corresponds to different elementary kernels evaluated on a given pair of data.} \label{framework}
 \end{figure}
We notice that the deep kernel network in essence is a multi-layer perceptron (MLP), with nonlinear activation functions (see Fig. \ref{framework}). The difference is that the last layer is not designed for classification, rather than to deliver a similarity value. However, we can use standard back-propagation algorithm specific for MLP to optimize the weights in the deep kernel network. Considering the output kernel $\kappa_1^L$ (and its parameters  $\w$),  a slight variant (denoted as $J(\alpha,{\w})$) of the objective function in Eq.~\ref{eq:dual} is defined as 
\begin{equation}
\begin{array}{lll}
  \displaystyle & \displaystyle \min_{\alpha \geq 0,b,\w} &
 \displaystyle   \frac{1}{2} \sum_{i,j=1}^n \alpha_i \alpha_j  \y_i \y_j \kappa_1^{L}(\mathbf{x}_i,\mathbf{x}_j) \\ 
      \displaystyle                 &  & \ \ \ \ +  \displaystyle  C  \sum_{i=1}^n \ell(\y_i,\sum_{j=1}^n \alpha_j \y_j \kappa_1^{L}(\mathbf{x}_i,\mathbf{x}_j) + b) \\
                & \textrm{s.t.}   & \displaystyle \sum_{q=1}^{n_{l-1}} \w_{q,p}^{l-1}  = 1, \ \ \ \w_{q,p}^{l-1} \geq 0,  \\
                 &  &   l \in \{2,\dots, L\}, \ \ \ p \in \{ 1,\dots,n_l\}.    
\end{array}
\label{eq:dual2}
\end{equation}

Note that this objective function is now optimized w.r.t both $\alpha$ and $\w$ (the parameters of the deep multiple kernels). Assuming that the computation of gradients of the objective function $J$ w.r.t the output kernel $\kappa^{(L)}_1$ (i.e.~$\frac{\partial J}{\partial \kappa^{L}_1(.,.)}$) is tractable; according to the chain rule, the corresponding gradients w.r.t coefficients $\w$ are computed, and then used to update these weights using gradient descent (see also extra details in section \ref{Opt}). Note also that the above problem is not convex anymore (w.r.t $\alpha$, $\beta$ when taken jointly), however, releasing convexity makes it possible to explore a larger set of possible solutions resulting into a better estimator as also discussed subsequently and also in experiments. 
\begin{figure*}[hpbt]
\resizebox{1.02\textwidth}{!}{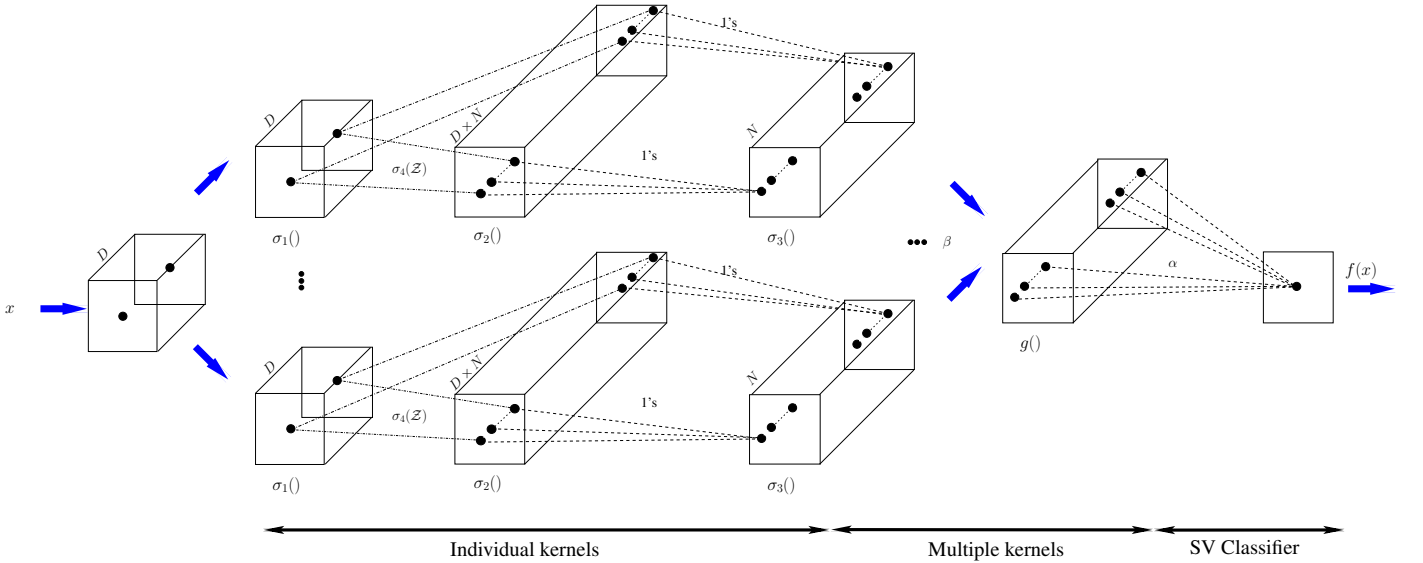}
\caption{This figure shows the architecture of our deep net including individual and deep kernel evaluations as well as SVM classification.  The first layer is fed with the input vector $\x$ (whose dimension is $D$) which is transferred to different branches;   each one  corresponds to one individual kernel (linear, polynomial, histogram intersection, etc.). In the second layer, the $\sigma_1$ activation is first applied to all the dimensions of $\x$, then each dimension of the resulting activated signal $\sigma_1(\x)$ is multiplied, in the third layer, by the $N$ (re-parameterized) weights $\{\sigma_4({\cal Z})\}$ (as shown in Eq~\ref{expansion0}) prior to apply the $\sigma_2$ activation; here $N$ corresponds to the number of virtual support  vectors in ${\cal Z}$ and these weights (or virtual support vectors) are shared through different branches (i.e., individual kernels). In the fourth layer, the results of the previous one are pooled across dimensions resulting into $N$ kernel values. Note that these $N$ kernel values are activated by $\sigma_3()$ and fed to the deep multiple kernel MLP,  in the fifth layer, resulting into $N$ deep multiple kernel values which are linearly combined in the last layer for classification (as shown in Eq.~\ref{eq:decision}). } \label{fig11}
\end{figure*}

\section{Deep total variation SVM networks}\label{section3}

Considering the primal form in Eq.~\ref{eq:hard-margin} (and Eq.~\ref{representer}); for a given $N$ (targeted number of   support vectors), we define a more general class of solutions as
\begin{equation}\label{representer2}
  {\bf w}=\sum_{j=1}^N \alpha_j \phi( \mathbf{z}_j).
\end{equation}
In the above form, ${\bf w}$ is not written in the span of the fixed training set $\cal L$ but in the span of a more general set ${\cal Z}=\{\z_j\}_{j=1}^N \subset \mathbb{R}^p$, referred to as {\it virtual support vectors}, which varies together with $\{\alpha_j\}_j$. As the labels of  $\cal Z$ are unknown, these labels are implicitly embedded into $\{\alpha_j\}_j$, so the latter are not constrained to be positive anymore. With this more general setting of ${\bf w}$ (i.e., $\{\alpha_j\}_j$ and  $\{\z_j\}_j$), one may define a larger set of possible solutions, and thereby obtain a {\it better} universal estimator; at least because the set of solutions defined by Eq.~\ref{representer} is included in Eq.~\ref{representer2} (while the converse is not true). \\

\indent Considering this variant of  ${\bf w}$, the objective function (\ref{eq:dual2}) can be rewritten as
\begin{equation}
\begin{array}{lll}
  \displaystyle & \displaystyle \min_{\alpha,b,{\cal Z},\w} &
 \displaystyle   \frac{1}{2} \sum_{i,j=1}^N \alpha_i \alpha_j  \kappa_1^{L}(\z_i,\z_j) \\ 
     &  & \displaystyle \ \ \ \ + C  \sum_{i=1}^n \ell(\y_i,\sum_{j=1}^N \alpha_j \kappa_1^{L}(\mathbf{x}_i,\z_j) + b) \\
                & \textrm{s.t.}   & \displaystyle \sum_{q=1}^{n_{l-1}} \w_{q,p}^{l-1}  = 1, \ \ \ \w_{q,p}^{l-1} \geq 0,  \\
                 &  &  l \in \{2,\dots, L\}, \ \ \ p \in \{ 1,\dots,n_l\}, 
\end{array}
\label{eq:dual3}
\end{equation}
and the decision function becomes
\begin{equation}
  \begin{array}{lll}
    f(\x_i) =  \displaystyle \sum_{j=1}^N \alpha_j \kappa_1^L(\x_i,\z_j) + b. 
\end{array}
\label{eq:decision}
\end{equation}

The left-hand side term of the  objective function (\ref{eq:dual3}) acts as a regularizer (which is now totally independent from the training set $\cal L$) while the right-hand side term  still corresponds to the hinge loss. With this new constrained minimization problem all the parameters are allowed to vary including the virtual support vectors $\cal Z$, together with $\alpha$ and $b$ as well as the mixing parameters (in $\beta$) of the deep multiple kernels. In contrast to standard non-parametric SVMs (as well as their multiple kernel variants), this formulation is totally parametric, which means, that the decision function (once trained on a given $\cal L$) is defined using a fixed-length set of trained parameters $\cal Z$, $\beta$, $\alpha$ and $b$. Note that the complexity of evaluating (\ref{eq:decision}) scales linearly w.r.t the size of  $\cal S$ while for standard (non-parametric) SVM this complexity is quadratic. As shown in the following section, one may consider a deep net architecture in order to effectively and efficiently train and evaluate the model in Eq.~\ref{eq:decision}.  In the remainder of this paper, this model will be referred to as {\it total variation SVM}; as shown later in experiments, this model is highly flexible and shows superior performances compared to individual and multi-kernel SVMs as well as their deep variants. 

\begin{table*}
  \begin{center}
    \resizebox{1.02\textwidth}{!}{
  \begin{tabular}{cc||c|cccc}
      & & $k(\x,\z)$ &  $\sigma_1(t)$ & $\sigma_2(t)$ & $\sigma_3(t)$ & $\sigma_4(t)$  \\
    \hline
    \hline 
     \multirow{4}{*}{\rotatebox{45}{\tiny Inner product based}} &  Linear  & $\langle \x,\z \rangle$ & $t$  & $t$  & $t$  & $t$    \\  
    & Polynomial  &$\langle \x,\z \rangle^p$ & $t$ &  $t$  & $t^p$ & $t$  \\  
    &Sigmoid & $\frac{1}{1+\exp(-\beta \langle \x,\z\rangle )}$ & $t$ & $t$  & $\frac{1}{1+\exp(-\beta t)}$ &  $t$  \\  
    &tanh  &$\tanh (a \langle \x,\z \rangle + b$) & $t$ & $t$  & $\tanh(a t + b)$ &  $t$  \\ 

  \hline
    \multirow{7}{*}{\rotatebox{45}{\tiny Distance based}}  &  Gaussian  & $\exp(-\beta \|\x-\z \|^2)$ & $\exp(t)$  & $\log(t)^2$& $\exp(-\beta t)$ & $\exp(-t)$  \\  
   & Laplacian  &$\exp(-\beta \|\x-\z \|)$& $\exp(t)$  & $\log(t)^2$& $\exp(-\beta \sqrt{t})$ & $\exp(-t)$  \\ 
    &Power & $-\|\x-\z\|^p$ &$\exp(t)$  & $\log(t)^2$& $-t^{p/2}$ & $\exp(-t)$  \\ 
    &Multi-quadratic & $\sqrt{\|\x-\z\|^2+b^2}$ & $\exp(t)$  & $\log(t)^2$ & $\sqrt{t+b^2}$ &  $\exp(-t)$   \\ 
    & Inverse Multi-quadratic  &$\frac{1}{\sqrt{\|\x-\z\|^2+b^2}}$ & $\exp(t)$  & $\log(t)^2$ & $\frac{1}{\sqrt{t+b^2}}$  &   $\exp(-t)$  \\ 
    &Log & $-\log(\|\x-\z\|^p+1)$  & $\exp(t)$  & $\log(t)^2$ & $-\log(t^{p/2}+1)$ &  $\exp(-t)$     \\ 
    & Cauchy  & $\frac{1}{1+\frac{\|\x-\z\|^2}{\sigma^2}}$  & $\exp(t)$  & $\log(t)^2$ & $\frac{1}{1+\frac{t}{\sigma^2}}$ &  $\exp(-t)$   \\ 
    \hline
   &  Histogram intersection & $\sum_d \min(\x_{.,d},\z_{.,d})$ & $\exp(\exp(-\beta(1-t)))$& $\frac{1}{\beta} \log(\log(t))+1$ & $t$  & $\sigma_1(t)$
  \end{tabular}}
\end{center}
\caption{This table shows the setting of $\sigma_1$, $\sigma_2$, $\sigma_3$, $\sigma_4$ for different kernel functions. Note that the best parameters of these individual kernels are set using cross validation.} \label{taxi}
\end{table*}

\subsection{Neural consistency and   architecture design}

In contrast to decision functions defined with the primal parameters $\bf w$, the one in Eq.~\ref{eq:decision} cannot be straightforwardly evaluated using standard neural units\footnote{i.e., those based on standard perceptron (inner product operators) followed by nonlinear activations.} as input kernels $\{\kappa_q^1\}$ in Eq.~\ref{eq0} may have general forms. Hence, modeling Eq.~\ref{eq:decision} requires a careful design; our goal in this paper, is not to change the definition of neural units, but instead to adapt Eq.~\ref{eq:decision} in order to make it consistent with the usual definition of neural units. In what follows, we introduce the overall architecture associated to the decision function $f(.)$ (and also $\kappa_1^L$) for different input kernels $\{\kappa_q^1\}_q$ including linear, polynomial, gaussian and histogram intersection as well as a more general class of kernels (see for instance \cite{cortes2009learning,maji2012efficient}).
\def\x{{\bf x}}
\def\y{{\bf y}}
\begin{definition}[Neural consistency] Let $\x_{.,d}$ (resp. $\z_{.,d}$) denote the $d^{th}$ dimension of a vector $\x$ (resp. $\z$). For a given (fixed or learned) $\z$, a kernel $\kappa$ is referred to as ``neural-consistent'' if
\begin{equation}\label{expansion0} 
  \kappa(\x,\z) = \sigma_3\big(\sum_d \sigma_2(\sigma_1(\x_{.,d}).\omega_{d})\big),
  \end{equation} 
with   $\omega_{d}=\sigma_4(\z_{.,d})$ and being  $\sigma_1$, $\sigma_2$, $\sigma_3$, $\sigma_4$ any arbitrary real-valued activation functions.   
\end{definition}
Considering the above definition, it is easy to see that deep kernels defined in Eq.~\ref{eq0} are neural consistent provided that their input kernels are also neural consistent; the latter include (i) the linear $\langle \x,\z\rangle$, (ii) the polynomial $\langle \x,\z\rangle^p$, (iii) the hyperbolic tangent  $\tanh \langle \x,\z\rangle$, (iv) the gaussian $\exp(-\beta \|\x-\z\|^2)$ and (v) the histogram intersection $\sum_d \min(\x_{.,d},\z_{.,d})$. Neural consistency is straightforward for inner product-based kernels (namely linear, polynomial and tanh) while for shift invariant kernels such as the gaussian, one may write $\exp(-\beta \|\x-\z\|_2^2) = \sigma_3\big(\sum_d \sigma_2(\sigma_1(\x_{.,d}).\omega_d)\big)$ with $\sigma_1(.)=\exp(.)$, $\sigma_2(.) = \log(.)^2$, $\sigma_3(.)= \exp(-\beta (.))$ and $\omega_d=\exp(-\z_{.,d})$. For histogram intersection, it is easy to see that $\sum_d \min(\x_{.,d},\z_{.,d})  =  \sum_d 1-\max(1-\x_{.,d},1-\z_{.,d})$ and one may obtain (for  a sufficiently large $\beta$) $\sum_d 1-\max(1-\x_{.,d},1-\z_{.,d}) \approx \sigma_3\big(\sum_d \sigma_2(\sigma_1(\x_{.,d}).\omega_d)\big)$ using 
$\sigma_1(.)=\exp(\exp(-\beta(1-(.))))$, $\sigma_2(.) = \frac{1}{\beta} \log(\log(.))+1$, $\sigma_3(.)= (.)$ and $\omega_d=\sigma_1(\z_{.,d})$. \\
Following  the above example, neural consistent kernels (including linear, polynomial, gaussian, histogram intersection) can be expressed using the deep net architecture shown in Fig.~\ref{fig11}. Neural consistency can be extended to other shift invariant kernels including:  multi-quadratic, inverse multi-quadratic, power, log, Cauchy,  Laplacian, etc. (see for instance \cite{genton2001classes} for a taxonomy of the widely used functions in kernel machines; see also table \ref{taxi} for the setting of $\sigma_1$, $\sigma_2$, $\sigma_3$, $\sigma_4$ for different kernels).

\begin{figure*}[hpbt]
  \begin{center}
    \centerline{\scalebox{0.39}{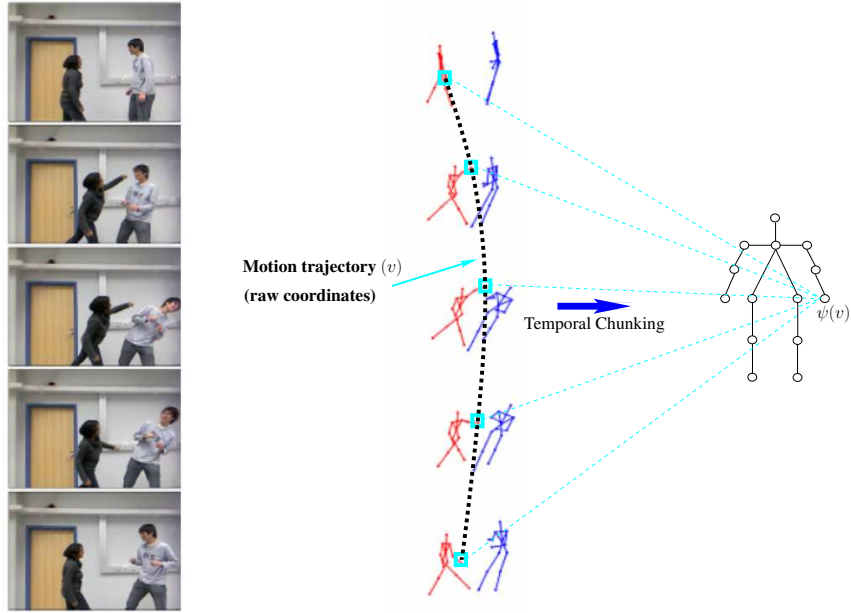}}
\caption{This figure shows the whole keypoint  tracking and description process on the motion stream.} \label{fig1}
\end{center}
\end{figure*}

\def\betaa{{\hat{\w}}}
 \begin{figure}[tb]
 \begin{minipage}[b]{1\linewidth}
   \centering
   \includegraphics[width=9cm]{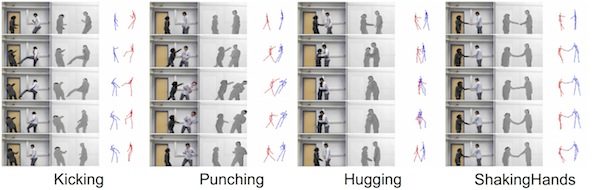}
 \end{minipage}
 \caption{This figure shows a sample of videos and skeletons associated to different action categories taken from the SBU dataset; \newline see {\scriptsize https://www3.cs.stonybrook.edu/$\sim$kyun/research/kinect\_interaction/index.html}.} \label{examples}
 \end{figure}

\subsection{Optimization}\label{Opt}
All the parameters of the network (including the virtual support vectors, the mixing parameters of the multiple kernels and the weights of the SVMs) are learned end-to-end using back-propagation and stochastic gradient descent. However, as the mixing parameters $\beta$ are constrained, we consider a slight variant in order to implement these constraints.\\ 
\indent From proposition~\ref{prop0} (see appendix), provided that i) the input kernels are conditionally positive definite (c.p.d), ii) the activation function $g$ preserves the c.p.d (as leaky-ReLU\footnote{This function is used in practice as it preserves negative kernel values (in contrast to ReLU).}) and iii) the weights  $\{\w_{q,p}^{l-1}\}$ (written for short as $\{\w_{q,p}^{l}\}$ in the remainder of this section)  are positive, all the resulting multiple kernels in Eq.~\ref{eq0} will also be c.p.d and admit equivalent positive definite kernels (following proposition~\ref{prop1} in appendix), and thereby max-margin SVM can be achieved. Note that conditions (i) and (ii) are satisfied by construction while condition (iii) requires adding equality and inequality constraints to Eq.~\ref{eq0}, i.e., $\w_{q,p}^{l} \in [0,1]$ and $\sum_{q}  \w_{q,p}^{l}=1$.  In order to implement these constraints, we consider a re-parametrization in Eq.~\ref{eq0} as $\w_{q,p}^{l}= h(\betaa_{q,p}^{l})\slash {\sum_q h(\betaa_{q,p}^{l})}$ for some $\{\hat{\w}_{q,p}^{l}\}$
with $h$ being strictly monotonic real-valued (positive) function and this allows free settings of the parameters  $\{\hat{\w}_{q,p}^{l}\}$ during optimization while guaranteeing $\w_{q,p}^{l} \in [0,1]$ and $\sum_{q}  \w_{q,p}^{l}=1$. During back-propagation, the gradient of the loss $J$ (now w.r.t $\betaa$'s) is updated using the chain rule as
\begin{equation}
  \begin{array}{lll}
 \displaystyle   \frac{\partial J}{\partial \betaa_{q,p}^{\l}} &=& \displaystyle \frac{\partial J}{\partial \w_{q,p}^{l}} . \frac{\partial \w_{q,p}^{l}}{\partial \betaa_{q,p}^{l}}  \\ 
              \ \ \ \ \textrm{with}   & & \ \ \ \displaystyle  \frac{\partial \w_{q,p}^{l}}{\partial \betaa_{q,p}^{\l}} = \displaystyle \frac{h'(\betaa_{q,p}^{l})h(\sum_{r\neq q } \betaa_{r,p}^{l})}{(h( \betaa_{q,p}^{l})+h(\sum_{r\neq q} \betaa_{r,p}^{l}))^2},
                                                                                                                                                                                         \end{array}
                                                                                                                                                                                           \end{equation} 
\noindent in practice $h(.)=\exp(.)$ and   $\frac{\partial J}{\partial \w_{q,p}^{l}}$ is obtained from layerwise gradient backpropagation (as already integrated in standard deep learning tools including PyTorch and TensorFlow).  Hence,  $\frac{\partial J}{\partial \betaa_{q,p}^{l}}$ is obtained by multiplying the original gradient $\frac{\partial J}{\partial {\w}_{q,p}^{l}}$ by $\frac{\exp(\sum_{r} \betaa_{r,p}^{l})}{(\exp( \betaa_{q,p}^{l})+\exp(\sum_{r\neq q} \betaa_{r,p}^{l}))^2}$. 

\section{Experiments}\label{section4}

We evaluate the performance of our total variation SVM on the challenging task of action recognition, using the   SBU kinect dataset \cite{SBU12}. The latter is an interaction dataset acquired  using the Microsoft kinect sensor; it includes in total 282 video sequences belonging to 8 categories:  ``approaching'', ``departing'', ``pushing'', ``kicking'', ``punching'', ``exchanging objects'', ``hugging'', and ``hand shaking'' with variable duration, viewpoint   changes and    interacting individuals (see examples in  Fig. \ref{examples}). In all these experiments, we use the same evaluation protocol as the one suggested in \cite{SBU12} (i.e., train-test split) and we report the average accuracy over all the classes of actions.

\begin{table*}
  \begin{center}
    \resizebox{0.99\textwidth}{!}{
  \begin{tabular}{c||c|cccccc}
\backslashbox{kernels}{classifiers}  & Standard SVM with   SVs &  \multicolumn{6}{c}{Total variation SVM w.r.t different \# of learned SVs} \\
   &    fixed/taken from the training set &  10& 50 & 100 & 150 & 200& 250 \\
    \hline
    \hline
    Linear  &   81.5385  &90.7692  & 92.3077 & 92.3077 & 93.8462 & 89.2308  & 92.3077 \\       
    Polynomial  &84.6154 &89.2308 & 92.3077 &92.3077 & 90.7692 &90.7692 &  90.7692 \\    
    Sigmoid   & 83.0769 &92.3077 & 90.7692 &92.3077 & 92.3077&92.3077  & 92.3077 \\
    tanh   &  83.0769 & 92.3077 & 93.8462  &  90.7692 & 92.3077 & 93.8462 & 93.8462 \\
    \hline
    Gaussian  &  86.1538 & 90.7692 &  92.3077 & 92.3077  & 92.3077 & 92.3077 &  92.3077\\
    Laplacian  &  84.6154 & 83.0769 &80.0000 &  81.5385 & 87.6923 & 92.3077 & 86.1538   \\
    Power & 86.1538 & 81.5385 &92.3077 &93.8462 & 95.3846 & 95.3846 & 95.3846 \\
    Multi-quadratic & 86.1538 &81.5385 & 90.7692 & 93.8462 & 95.3846 &93.8462 & 95.3846  \\
    Inverse Multi-quadratic & 83.0769 & 90.7692 & 93.8462 & 93.8462 & 93.8462 & 93.8462& 93.8462  \\
    Log &  87.6923 & 78.4615 & 89.2308 & 92.3077 & 93.8462 & 95.3846&  95.3846 \\
    Cauchy  & 86.1538 &92.3077 &  93.8462  & 93.8462 & 93.8462 &  93.8462& 93.8462 \\
    \hline
    Histogram Intersection & 86.1538  & 92.3077 & 92.3077 &  93.8462  & 93.8462 & 93.8462 & 95.3846 \\
    \hline
    \hline
      MKL (one layer)  & 89.2308 & 80.0000 &95.3846 &  93.8462&95.3846 & 93.8462 &  93.8462 \\ 
     MKL (two layers)  & 87.6923 &90.7692 & 95.3846 &  93.8462& {\bf 96.9231} & 95.3846 &  95.3846 \\
     MKL (three layers)  & 89.2308 & 93.8462 &95.3846  &95.3846  & 93.8462      &95.3846 &  95.3846                                                                                                                                         
  \end{tabular}}
  \end{center}  
  \caption{This table shows a comparison of our TV SVM against standard SVM with different individual and multiple kernels.}\label{table1}
\end{table*}

\subsection{Video skeleton description}\label{graphc}
 \indent Given a  video $\V$ in SBU as a sequence of skeletons, each keypoint in these skeletons defines a labeled trajectory through successive frames (see Fig.~\ref{fig1}).   Considering a finite collection of trajectories $\{v_i\}_i$ in $\V$, we process each trajectory  using {\it temporal chunking}: first we split the total duration of a  video into $M$ equally-sized temporal chunks ($M=4$ in practice), then we assign  the keypoint  coordinates of  a given trajectory $v_i$  to the $M$ chunks (depending on their time stamps) prior to concatenate the averages of these chunks and this produces the description of $v_i$ denoted as $\psi(v_i)$ and the  final description of a given video $\V$ is $(\psi(v_1) \ \psi(v_2) \dots)$ following the same   order through trajectories. Hence, two trajectories $v_i$ and $v_j$,  with similar keypoint coordinates but arranged differently in time, will be considered as very different. Note that beside  being compact and discriminant, this temporal chunking gathers advantages --  while discarding drawbacks -- of two widely used families of techniques mainly {\it global averaging techniques} (invariant but less discriminant)  and  {\it frame resampling techniques} (discriminant but less invariant). Put differently, temporal chunking produces discriminant descriptions that preserve the temporal structure of trajectories while being {\it frame-rate} and {\it duration} agnostic.

\subsection{Performances and comparison} 

\indent  We trained our TV SVM  for 1000 epochs  with a batch size equal to $50$ and we set the learning rate (denoted $\nu$)  iteratively inversely proportional to the speed of change of the objective function in Eq. \ref{eq:dual3}; when this speed increases (resp. decreases),   $\nu$  decreases as $\nu \leftarrow \nu \times 0.99$ (resp. increases as $\nu \leftarrow \nu \slash 0.99$). All these experiments are run on a GeForce GTX 1070 GPU  device (with 8 GB memory) and no data augmentation is achieved.  Table~\ref{table1}  shows a comparison of action recognition performances, using TV SVM against different baselines involving individual kernels  with    support vectors (SVs) fixed/taken from the training set. We also show the results using deep multiple kernel learning (MKL).  \\ 

From all these results, we observe a clear and a consistent gain of TV SVM w.r.t all the individual kernel settings and their MKL combinations; this gain is further amplified when using deep MKL with only two layers.  However, these performances  stabilize  as  the depth of MKL increases since  the size of the training set is limited compared to the large number of training parameters in its underlying MLP.  These  performances also improve  quickly as the number of learned (virtual) support vectors $N$ increases, and this results from the flexibility of TV SVM which learns --- with few virtual support vectors ---  relevant representatives  of training and test data.  These performances   consistently improve for all the individual kernels (as well as their MKL combinations) and this is again explained by the modeling capacity of TV SVM. Indeed, the latter captures better the actual decision boundary while standard SVM (even when combined with MKL) is clearly limited when the fixed support vectors are biased (i.e., not sufficiently representative of the actual  distribution of  the data especially on small or mid-scale problems); hence, learning the MKL+SVM parameters (with fixed SVs) is not enough in order to recover from this bias.  In sum, the gain of TV SVM results from the {\it complementary aspects of the used individual kernels and also the modeling capacity of SVMs when the support vectors are allowed to vary}.\\

\noindent  Finally, we compare the classification performances   of our TV SVM against other related methods in action recognition  (on SBU) ranging from sequence based such as LSTM and GRU \cite{DeepGRU,GCALSTM,STALSTM} to deep graph (no-vectorial) methods based on spatial and spectral convolution \cite{kipf17,SGCCONV19,Bresson16}. From the results in table \ref{compare},  TV SVM brings a substantial gain w.r.t state of the art methods, and provides comparable results with the best vectorial methods on SBU (see also \cite{sahbibmvc2019}).

 \begin{table}[!htb]
  \begin{center}
\resizebox{0.35\linewidth}{!}{
\begin{adjustbox}{angle=-90}
\setlength\tabcolsep{2.4pt}
  \begin{tabular}{c||ccccccccccccccccccc}
   \rotatebox{90}{Perfs} &     \rotatebox{90}{90.00} &  \rotatebox{90}{96.00} &  \rotatebox{90}{94.00}&  \rotatebox{90}{96.00}&   \rotatebox{90}{49.7 }&  \rotatebox{90}{80.3 }&  \rotatebox{90}{86.9 }&  \rotatebox{90}{83.9 }&  \rotatebox{90}{80.35 }&  \rotatebox{90}{90.41}&   \rotatebox{90}{93.3 } &  \rotatebox{90}{90.5}&   \rotatebox{90}{91.51}&  \rotatebox{90}{94.9}&  \rotatebox{90}{97.2}&  \rotatebox{90}{95.7}&  \rotatebox{90}{93.7 } &  \rotatebox{90}{{96.92} }\\  
     &  &  &  &  &  &  &  &  &  &  &        &  &  &  &  &  &  &  &     \\
     \rotatebox{90}{Methods} &    \rotatebox{90}{ GCNConv \cite{kipf17}} & \rotatebox{90}{ArmaConv \cite{ARMACONV19}} & \rotatebox{90}{ SGCConv \cite{SGCCONV19}} & \rotatebox{90}{ ChebyNet \cite{Bresson16}}& \rotatebox{90}{  Raw coordinates  \cite{SBU12}} & \rotatebox{90}{Joint features \cite{SBU12}} & \rotatebox{90}{Interact Pose \cite{InteractPose}} & \rotatebox{90}{CHARM \cite{CHARM15}} & \rotatebox{90}{ HBRNN-L \cite{HBRNNL15}} & \rotatebox{90}{Co-occurence LSTM \cite{CoOccurence16}} & \rotatebox{90}{ ST-LSTM \cite{STLSTM16}}  & \rotatebox{90}{ Topological pose ordering\cite{velocity2}} & \rotatebox{90}{ STA-LSTM \cite{STALSTM}} & \rotatebox{90}{ GCA-LSTM \cite{GCALSTM}} & \rotatebox{90}{ VA-LSTM  \cite{VALSTM}} & \rotatebox{90}{DeepGRU  \cite{DeepGRU}} & \rotatebox{90} {Riemannian manifold trajectory\cite{RiemannianManifoldTraject}}  &  \rotatebox{90}{Our best TV SVM}   \\  
      \end{tabular}
\end{adjustbox}
 }
 \caption{Comparison against state of the art methods.}   \label{compare}            
\end{center}
\end{table}

\section{Conclusion}\label{section5}
We introduced in this paper a novel deep total variation support vector machine that learns highly effective classifiers. The strength of our method resides in its ability to models different kernels and to learn their support vectors   resulting into better classification performances. Experiments conducted on the challenging action recognition task show the outperformance of this parametric SVM formulation against different baselines including non-parametric SVMs as well as the related work. 
As a future work, we are currently investigating the application of our method to other computer vision and pattern recognition tasks in order to further study the impact of this highly flexible model. 
\section*{Appendix} 

We consider, as $g$ in Eq.~(\ref{eq0}), the  leaky ReLU~\cite{softplus1,leaky}  activation function:  leaky ReLU  allows  learning {\it conditionally} positive definite (c.p.d) kernels. In what follows, we discuss  the sufficient conditions about the choices of the  input kernels, the parameters  $\{\w_{q,p}^{l-1}\}$ and the activation function that guarantee this c.p.d property.  
\begin{definition}[c.p.d kernels]\label{def0} 
A kernel $\kappa$ is c.p.d, iff $\forall {\bf x}_1,\dots,{\bf x}_n \in \mathbb{R}^p$, $\forall c_1,\dots, c_n \in \mathbb{R}$ (with $\sum_{i=1}^n c_i = 0$), we have $\sum_{i,j} c_i c_j \kappa({\bf x}_i,{\bf x}_j) \geq 0$.
\end{definition} 
From the above definition, it is clear that any p.d kernel is also c.p.d, but the converse is not true; this property is a weaker (but sufficient) condition in order to learn max margin SVMs~(see for instance \cite{ShaweTaylor2004}; see also the following proposition).
\begin{proposition}[Berg et al.\cite{berg84}]\label{prop1}
 Consider $\kappa$ and define $\hat{\kappa}$ with
\begin{eqnarray*}\label{eq0101}
  \hat{\kappa}(\x_i,\x_j)&=& {\kappa}(\x_i,\x_j) - {\kappa}(\x_i,\x_{n+1}) \\ 
                                     &  & -  {\kappa}(\x_{n+1},\x_j) +   {\kappa}(\x_{n+1},\x_{n+1})
\end{eqnarray*}
Then, $\hat{\kappa}$ is positive definite if and only if  ${\kappa}$ is c.p.d.
\end{proposition} 
\begin{proof} See Berg et al.\cite{berg84}. Now we derive our main result \end{proof} 
\begin{proposition}\label{prop0}
Provided that the input kernels $\{\kappa_q^{(1)}\}_{q}$ are c.p.d, and  $\{\w_{q,p}^{l-1}\}_{p,q,l}$ belong to the positive orthant of the parameter space; any combination  defined as  $g(\sum_q {\w}_{q,p}^{l-1} \ \kappa_q^{l-1})$ with  $g$ equal to leaky ReLU is also c.p.d.
\end{proposition} 
{
\begin{proof}
  
  \noindent Details of the first part of the proof, based on recursion, are omitted and result from the application of definition~(\ref{def0}) to $\kappa=\sum_q \w_{q,p}^{l-1} \ \kappa_q^{l-1}$ (for different values of $l$) while considering $\{\kappa_q^{1}\}_q$ c.p.d. Now we show the second part of the proof (i.e., if $\kappa$ is c.p.d, then $g(\kappa)$ is also c.p.d for leaky ReLU).\\
\noindent  For leaky ReLU, one may write $g(\kappa)=\log(\exp(a \kappa) +\exp(\kappa))$ with $0<a\ll 1$. Considering $\kappa$ c.p.d, and following proposition~(\ref{prop1}), one may define a positive definite $\hat{\kappa}$ and obtain $\forall \{c_i\}_i$, $\forall \{\x_i\}_i$, $\sum_{i,j=1}^n c_i c_j \exp(\kappa(\x_i,\x_j)) =(*)$ with
{

\begin{eqnarray*}
(*)  & = &  \exp(\kappa(\x_{n+1},\x_{n+1}))  \\ 
      &    &  . \sum_{i,j=1}^n (c_i \exp(\kappa(\x_i,\x_{n+1}))) .  (c_j \exp(\kappa(\x_{n+1},\x_j)))  \\
       &   &   \ \ \ \ \ \ \ \ \ \ \ \  . \exp(\hat{\kappa}(\x_i,\x_j))\\
  &\geq &0, 
 \end{eqnarray*}
 }
so $\exp(\kappa)$ is also positive definite. One may also rewrite $g$ as
\begin{eqnarray}\label{eq1112}
g(\kappa)=a \ \kappa + \log(1 +\exp((1-a) \ \kappa)).
\end{eqnarray}
Since $\exp(\kappa)$ is positive definite, it follows that $(1+\exp((1-a) \ \kappa))^{\alpha}$ is also positive definite for any arbitrary $\alpha >0$ and $0<a\ll1$  so from \cite{Shoenberg38}, $\log(1 +\exp((1-a) \ \kappa))$ is c.p.d and so is $g(\kappa)$; the latter results from the closure of the c.p.d with respect to the sum.   
\begin{flushright}$\blacksquare$ \end{flushright}
\end{proof}  
}

\end{document}

%% file: archi3.pdf_t
\begin{picture}(0,0)%
\includegraphics{archi3.pdf}%
\end{picture}%
\setlength{\unitlength}{4144sp}%
\begingroup\makeatletter\ifx\SetFigFont\undefined%
\gdef\SetFigFont#1#2#3#4#5{%
  \reset@font\fontsize{#1}{#2pt}%
  \fontfamily{#3}\fontseries{#4}\fontshape{#5}%
  \selectfont}%
\fi\endgroup%
\begin{picture}(26688,10686)(-2309,-10960)
\put(2746,-7351){\rotatebox{45.0}{\makebox(0,0)[lb]{\smash{{\SetFigFont{17}{20.4}{\rmdefault}{\mddefault}{\updefault}{\color[rgb]{0,0,0}$D$}%
}}}}}
\put(11971,-7576){\rotatebox{45.0}{\makebox(0,0)[lb]{\smash{{\SetFigFont{17}{20.4}{\rmdefault}{\mddefault}{\updefault}{\color[rgb]{0,0,0}$N$}%
}}}}}
\put(6211,-10816){\makebox(0,0)[lb]{\smash{{\SetFigFont{25}{30.0}{\rmdefault}{\mddefault}{\updefault}{\color[rgb]{0,0,0}Individual kernels}%
}}}}
\put(2746,-2626){\rotatebox{45.0}{\makebox(0,0)[lb]{\smash{{\SetFigFont{17}{20.4}{\rmdefault}{\mddefault}{\updefault}{\color[rgb]{0,0,0}$D$}%
}}}}}
\put(2817,-4829){\makebox(0,0)[lb]{\smash{{\SetFigFont{20}{24.0}{\rmdefault}{\mddefault}{\updefault}{\color[rgb]{0,0,0}$\sigma_1()$}%
}}}}
\put(5104,-3486){\makebox(0,0)[lb]{\smash{{\SetFigFont{17}{20.4}{\rmdefault}{\mddefault}{\updefault}{\color[rgb]{0,0,0}$\sigma_4({\cal Z})$}%
}}}}
\put(11971,-2851){\rotatebox{45.0}{\makebox(0,0)[lb]{\smash{{\SetFigFont{17}{20.4}{\rmdefault}{\mddefault}{\updefault}{\color[rgb]{0,0,0}$N$}%
}}}}}
\put(6301,-2941){\rotatebox{45.0}{\makebox(0,0)[lb]{\smash{{\SetFigFont{17}{20.4}{\rmdefault}{\mddefault}{\updefault}{\color[rgb]{0,0,0}$D \times N$}%
}}}}}
\put(11381,-712){\makebox(0,0)[lb]{\smash{{\SetFigFont{17}{20.4}{\rmdefault}{\mddefault}{\updefault}{\color[rgb]{0,0,0}1's}%
}}}}
\put(12274,-4823){\makebox(0,0)[lb]{\smash{{\SetFigFont{20}{24.0}{\rmdefault}{\mddefault}{\updefault}{\color[rgb]{0,0,0}$\sigma_3()$}%
}}}}
\put(6663,-4785){\makebox(0,0)[lb]{\smash{{\SetFigFont{20}{24.0}{\rmdefault}{\mddefault}{\updefault}{\color[rgb]{0,0,0}$\sigma_2()$}%
}}}}
\put(9851,-3247){\makebox(0,0)[lb]{\smash{{\SetFigFont{17}{20.4}{\rmdefault}{\mddefault}{\updefault}{\color[rgb]{0,0,0}1's}%
}}}}
\put(2817,-9554){\makebox(0,0)[lb]{\smash{{\SetFigFont{20}{24.0}{\rmdefault}{\mddefault}{\updefault}{\color[rgb]{0,0,0}$\sigma_1()$}%
}}}}
\put(5104,-8211){\makebox(0,0)[lb]{\smash{{\SetFigFont{17}{20.4}{\rmdefault}{\mddefault}{\updefault}{\color[rgb]{0,0,0}$\sigma_4({\cal Z})$}%
}}}}
\put(6301,-7666){\rotatebox{45.0}{\makebox(0,0)[lb]{\smash{{\SetFigFont{17}{20.4}{\rmdefault}{\mddefault}{\updefault}{\color[rgb]{0,0,0}$D \times N$}%
}}}}}
\put(11381,-5437){\makebox(0,0)[lb]{\smash{{\SetFigFont{17}{20.4}{\rmdefault}{\mddefault}{\updefault}{\color[rgb]{0,0,0}1's}%
}}}}
\put(12274,-9548){\makebox(0,0)[lb]{\smash{{\SetFigFont{20}{24.0}{\rmdefault}{\mddefault}{\updefault}{\color[rgb]{0,0,0}$\sigma_3()$}%
}}}}
\put(6663,-9510){\makebox(0,0)[lb]{\smash{{\SetFigFont{20}{24.0}{\rmdefault}{\mddefault}{\updefault}{\color[rgb]{0,0,0}$\sigma_2()$}%
}}}}
\put(9851,-7972){\makebox(0,0)[lb]{\smash{{\SetFigFont{17}{20.4}{\rmdefault}{\mddefault}{\updefault}{\color[rgb]{0,0,0}1's}%
}}}}
\put(-2294,-6181){\makebox(0,0)[lb]{\smash{{\SetFigFont{20}{24.0}{\rmdefault}{\mddefault}{\updefault}{\color[rgb]{0,0,0}$x$}%
}}}}
\put(19939,-5294){\makebox(0,0)[lb]{\smash{{\SetFigFont{17}{20.4}{\rmdefault}{\mddefault}{\updefault}{\color[rgb]{0,0,0}$\alpha$}%
}}}}
\put(15626,-4889){\makebox(0,0)[lb]{\smash{{\SetFigFont{17}{20.4}{\rmdefault}{\mddefault}{\updefault}{\color[rgb]{0,0,0}$\beta$}%
}}}}
\put(23315,-5451){\makebox(0,0)[lb]{\smash{{\SetFigFont{20}{24.0}{\rmdefault}{\mddefault}{\updefault}{\color[rgb]{0,0,0}$f(x)$}%
}}}}
\put(17104,-6840){\makebox(0,0)[lb]{\smash{{\SetFigFont{20}{24.0}{\rmdefault}{\mddefault}{\updefault}{\color[rgb]{0,0,0}$g()$}%
}}}}
\put(15346,-10815){\makebox(0,0)[lb]{\smash{{\SetFigFont{25}{30.0}{\rmdefault}{\mddefault}{\updefault}{\color[rgb]{0,0,0}Multiple kernels}%
}}}}
\put(20341,-10777){\makebox(0,0)[lb]{\smash{{\SetFigFont{25}{30.0}{\rmdefault}{\mddefault}{\updefault}{\color[rgb]{0,0,0}SV Classifier}%
}}}}
\put(-449,-5191){\rotatebox{45.0}{\makebox(0,0)[lb]{\smash{{\SetFigFont{17}{20.4}{\rmdefault}{\mddefault}{\updefault}{\color[rgb]{0,0,0}$D$}%
}}}}}
\end{picture}%

%% file: processing.pdf_t
\begin{picture}(0,0)%
\includegraphics{processing.pdf}%
\end{picture}%
\setlength{\unitlength}{4144sp}%
\begingroup\makeatletter\ifx\SetFigFont\undefined%
\gdef\SetFigFont#1#2#3#4#5{%
  \reset@font\fontsize{#1}{#2pt}%
  \fontfamily{#3}\fontseries{#4}\fontshape{#5}%
  \selectfont}%
\fi\endgroup%
\begin{picture}(12699,10296)(7471,-22192)
\put(11184,-16614){\makebox(0,0)[lb]{\smash{{\SetFigFont{17}{20.4}{\rmdefault}{\bfdefault}{\updefault}{\color[rgb]{0,0,0}(raw coordinates)}%
}}}}
\put(15482,-17048){\makebox(0,0)[lb]{\smash{{\SetFigFont{17}{20.4}{\rmdefault}{\mddefault}{\updefault}{\color[rgb]{0,0,0}Temporal Chunking}%
}}}}
\put(19969,-16881){\makebox(0,0)[lb]{\smash{{\SetFigFont{17}{20.4}{\rmdefault}{\bfdefault}{\updefault}{\color[rgb]{0,0,0}$\psi(v)$}%
}}}}
\put(11162,-16164){\makebox(0,0)[lb]{\smash{{\SetFigFont{17}{20.4}{\rmdefault}{\bfdefault}{\updefault}{\color[rgb]{0,0,0}Motion trajectory $(v)$}%
}}}}
\end{picture}%